\documentclass[11pt]{article}

\usepackage[final]{acl}

\usepackage{times}
\usepackage{latexsym}

\usepackage[T1]{fontenc}

\usepackage[utf8]{inputenc}

\usepackage{microtype}

\usepackage{inconsolata}

\usepackage{hyperref}
\usepackage{pifont}
\usepackage{multirow}
\usepackage{graphicx}
\usepackage{float}
\usepackage[utf8]{inputenc} 
\usepackage[T1]{fontenc}    
\usepackage{url}            
\usepackage{booktabs}       
\usepackage{amsfonts}       
\usepackage{nicefrac}       
\usepackage{microtype}      
\usepackage{xcolor}         
\usepackage{pifont}
\usepackage{multirow}
\usepackage{makecell}
\usepackage{amsmath}
\usepackage{tcolorbox}
\usepackage{colortbl}
\usepackage{ulem}
\usepackage{wrapfig}
\usepackage{subcaption}
\usepackage{color}
\usepackage{enumitem}
\definecolor{lightgray}{gray}{0.9}

\usepackage{algorithm}
\usepackage{algorithmic}
\usepackage{amssymb}
\usepackage{amsthm}
\newtheorem{theorem}{Theorem}

\newtheorem{definition}{Definition}

\newtheorem{assumption}{Assumption}
\newtheorem{corollary}{Corollary}
\newtheorem{remark}{Remark}

\title{Online Self-Weighted Fine-Tuning}

\author{
Haiquan Wen\textsuperscript{1,*}
\quad
Yiwei He\textsuperscript{1,*}
\quad
Bei Peng\textsuperscript{2}
\quad
Guangliang Cheng\textsuperscript{1,\ensuremath{\dagger}}
\\
\textsuperscript{1}University of Liverpool
\quad
\textsuperscript{2}University of Sheffield
\\
{
\small
\textsuperscript{*}Equal contribution.
\quad
\textsuperscript{\ensuremath{\dagger}}\textbf{Correspondence:} {guangliang.cheng@liverpool.ac.uk}
}
}

\begin{document}

\maketitle

\begin{abstract}
Standard supervised fine-tuning (SFT) assigns the same explicit loss weight to every expert demonstration, regardless of the model's changing competence over training queries. Reinforcement learning (RL) based methods adapt update strength using model-generated rollouts, but often require substantially more sampling and can be unstable on hard tasks.
We propose \textbf{Online Self-Weighted Fine-Tuning (OSW-FT)}, a simple method that augments SFT with online, trajectory-level weighting. For each query, OSW-FT estimates the model's current success rate using a small number of inference-only rollouts and rescales the standard SFT loss accordingly. The optimization direction remains anchored to the expert trajectory, while the update magnitude adapts online.
For binary-verifiable reasoning, we connect this weighting to SFT and RL at the gradient level, inspired by variance-reduction principles. The resulting estimator is unbiased for the exact OSW-FT surrogate update for any finite rollout count, and we analyze convergence with respect to the corresponding surrogate objective. Evaluated across Qwen3 series ranging from 0.6B to 4B on multiple challenging benchmarks (e.g., AIME), OSW-FT consistently improves over SFT on small-to-medium scale models.
OSW-FT offers a favorable compute-performance trade-off as a practical approach for fine-tuning small-to-medium LLMs on binary-verifiable reasoning tasks with only \textbf{2 online rollouts}.

\end{abstract}

\section{Introduction}
\label{sec:intro}

Post-training a LLM raises a deceptively simple question: \textit{how much should each training sample update the model at its current stage of training?} This question is especially consequential for verifiable reasoning tasks, where the model's competence on individual training queries can vary substantially and evolves throughout training. Some queries may already be reliably solved by the model and provide limited additional learning signal, while others remain near its capability frontier and warrant concentrated optimization effort.
Standard Supervised Fine-Tuning (SFT) \citep{ouyang2022training}, however, assigns a uniform sample weight to every expert demonstration throughout training, treating already-mastered and unresolved queries identically. While this design is simple and stable, it allocates optimization effort coarsely across examples with very different learning value.

Reinforcement learning (RL) based post-training methods \citep{guo2025deepseek, schulman2017proximal} address this issue from a different direction. By computing advantages from model-generated rollouts, they adapt update strength to the model's current behavior and can focus learning on more informative examples.
This adaptivity is appealing, but it comes with substantial online sampling cost and optimization instability, especially on hard reasoning tasks where failing to discover any correct trajectories entirely deprives the model of a viable learning direction.
SFT avoids this failure mode by always providing a high-quality expert trajectory as the supervision target, but lacks a mechanism to modulate how strongly that trajectory should be learned from. This suggests a natural middle ground: preserve the stable, expert-anchored optimization direction of SFT, while introducing a lightweight online signal that adjusts update magnitude according to the model's current competence.

We propose \textbf{Online Self-Weighted Fine-Tuning (OSW-FT)}, a simple post-training method for verifiable reasoning that changes how much to learn from each expert trajectory without changing what to learn from. For each training query, OSW-FT performs a small number of inference-only rollouts to estimate the model's current pass rate $\hat{p}(q)$, and reweights the standard SFT loss by a factor of $(1 - \hat{p}(q))$. This weight is largest for queries the model consistently fails on, thereby concentrating optimization effort near the capability frontier, and naturally decreases toward zero as the model masters a query. Crucially, OSW-FT uses online rollouts only to estimate this scalar weight; it does not optimize on model-generated trajectories, does not perform token-level reward shaping, and does not modify the training distribution. This design targets settings with a limited online rollout budget while retaining expert-trajectory supervision.

For binary-verifiable reasoning, we construct an explanatory framework connecting this weighting to RL control variates at the gradient level. OSW-FT restores this quantity online through a Monte Carlo estimate from a small number of rollouts. Under our analysis, the resulting estimator is unbiased for the exact OSW-FT surrogate update for any finite rollout count $K$, and the convergence analysis applies to the corresponding surrogate objective under the stated assumptions. The method is naturally most useful when the training distribution remains challenging for the current model; when most training queries are already solved, the online weight correspondingly diminishes.

We evaluate OSW-FT on Qwen3 series \citep{yang2025qwen3technicalreport} models ranging from 0.6B to 4B parameters across AIME \citep{aime2025misc}, AMC \citep{amc2023misc}, MATH-500 \citep{lightman2023let}, and GPQA-Diamond \citep{rein2024gpqa}. Empirically, OSW-FT consistently improves over standard SFT on small-to-medium model scales. Compared with GRPO \citep{shao2024deepseekmath}, the results vary across models and metrics: OSW-FT performs favorably in several limited-rollout settings, whereas GRPO remains stronger in some configurations. OSW-FT is therefore not intended as a general replacement for RL; it targets binary-verifiable reasoning when the available online rollout budget is limited. Ablation studies further validate the benefit of continuous online success-rate weighting over alternative hard-thresholding or random heuristics.

    


In summary, this work makes the following core contributions: \textbf{First}, we introduce OSW-FT, an expert-trajectory-anchored method that modulates SFT updates using an online success-rate weight for verifiable reasoning. \textbf{Second}, we analyze finite-rollout unbiasedness and convergence with respect to the proposed surrogate update. \textbf{Third}, evaluations show that OSW-FT consistently improves over standard SFT on small-to-medium scales and provides a favorable compute–performance trade-off under a limited online rollout budget.

\section{Online Self-Weighted Fine-Tuning}
\label{sec:method}

Supervised Fine-Tuning (SFT) aligns language models by maximizing the log-likelihood of expert demonstrations. Given a dataset $\mathcal{D} = \{(q, o^*)\}$ of input queries $q$ and ground-truth output trajectories $o^*$, SFT trains the model $\pi_\theta$ by applying a uniform weight of 1 to the gradient of every ground-truth trajectory:


\begin{equation}
\begin{split}
    -\nabla_\theta \mathcal{J}_{\text{SFT}} &= \nabla_\theta \mathcal{L}_{\text{SFT}} \\
    &= -\mathbb{E}_{(q, o^*) \sim \mathcal{D}} \left[ 1 \cdot \nabla_\theta \log \pi_\theta(o^* \mid q) \right]
\end{split}
\end{equation}

From a classical estimation perspective, this uniform weighting is strictly principled as the standard maximum likelihood estimator over the data distribution. However, from an online optimization perspective, applying full gradient magnitude to queries the model has already mastered ($\pi_\theta(o^* \mid q) \to 1$) becomes statistically inefficient; it injects redundant gradient noise regarding the evolving policy and dampens the learning efficacy on unresolved frontier queries. To derive a principled variance-reduced gradient estimator that dynamically modulates this update intensity, we cast the problem within the framework of score function estimation.

\subsection{Score Function Decomposition and Optimal Control Variate}

For verifiable reasoning tasks, let $R(o) \in \{0,1\}$ denote the binary correctness of a generated trajectory. The alignment objective is to maximize expected correctness:

\begin{equation}
    \mathcal{J}(\theta) = \mathbb{E}_{q \sim \mathcal{D}} 
    \mathbb{E}_{o \sim \pi_\theta(\cdot \mid q)} [R(o)]
\end{equation}

Applying the log-derivative trick and subtracting a query-dependent control variate $c(q)$---which leaves the expected gradient unbiased since $\mathbb{E}_{o \sim \pi_\theta}[\nabla_\theta \log \pi_\theta(o \mid q)] = 0$---yields the standard variance-reduced score function estimator:


\begin{equation}
\begin{aligned}
    \nabla_\theta \mathcal{J} &= \mathbb{E}_{q \sim \mathcal{D}} \mathbb{E}_{o \sim \pi_\theta} \Big[ \\
    &\quad (R(o) - c(q)) \nabla_\theta \log \pi_\theta(o \mid q) \Big]
\end{aligned}
\end{equation}

We seek optimal control variate $c^*(q)$ that minimizes the MSE of the scalar learning signal:

\begin{equation}
    \min_{c(q)} \; \mathbb{E}_{o \sim \pi_\theta}\left[ (R(o) - c(q))^2 \right]
\end{equation}

Setting the derivative to zero yields the exact analytical solution $c^*(q) = \mathbb{E}_{o \sim \pi_\theta}[R(o) \mid q]$, which for binary tasks simplifies to the model's current success probability $p_s(q)$. Crucially, this optimal control variate is identical to the advantage of the oracle trajectory:


\begin{equation}
\begin{aligned}
    A(o^*, q) &\triangleq R(o^*) - \mathbb{E}_{o \sim \pi_\theta}[R(o) \mid q] \\
    &= 1 - p_s(q)
\end{aligned}
\end{equation}

The MSE-minimizing baseline and the unbiased advantage estimator for $o^*$ are therefore identical for binary rewards, providing a unified statistical justification for the signal magnitude derived below. Substituting $c^*(q) = p_s(q)$ and applying the Law of Total Expectation to decompose by reward outcome, we observe a \textbf{dual role} of $p_s(q)$: it represents both probability of generating correct trajectory ($\Pr[R=1]$) and determines its advantage magnitude:


\begin{equation}
\begin{aligned}
    \nabla_\theta \mathcal{J} = \mathbb{E}_{q \sim \mathcal{D}} \Big[
    &\underbrace{p_s(q)}_{\Pr[R=1]} \cdot \mathbb{E}_{o \sim \pi_\theta \mid R(o)=1} \Big[ \\
    &\quad \underbrace{(1 - p_s(q))}_{A(o^*,\, q)} \nabla_\theta \log \pi_\theta(o \mid q) \Big] \\
    +\, &(1 - p_s(q)) \cdot \mathbb{E}_{o \sim \pi_\theta \mid R(o)=0} \Big[ \\
    &\quad {(0-p_s(q))}\; \nabla_\theta \log \pi_\theta(o \mid q) \Big] \Big]
\end{aligned}
\end{equation}

\subsection{Oracle Substitution in SFT}

The equation highlights a fundamental tension: the positive learning signal is gated by the success probability $p_s(q)$.
In the low-resource or early-training regime where $p_s(q) \to 0$, \textbf{pure reinforcement learning methods} suffer from \textbf{gradient starvation}, as the agent cannot explore and discover valid reasoning paths from scratch. 

\paragraph{SFT as Deterministic Oracle Substitution.}
The essence of SFT is to bypass this exploration bottleneck by replacing the intractable generative probability with the absolute certainty of the dataset $\mathcal{D}$. By manually setting the probability of encountering a correct trajectory strictly to unity ($\Pr[R=1] \to 1$), SFT not only guarantees a continuous learning signal but also \textbf{mathematically nullifies the negative gradient term}:


\begin{equation}
\begin{aligned}
    &\nabla_\theta \mathcal{J} \approx \mathbb{E}_{q \sim \mathcal{D}} \Big[
    \underbrace{1}_{\text{Dataset Certainty}} \cdot \\
    &\mathbb{E}_{o \sim \pi_\theta \mid R(o)=1} \left[
    \underbrace{(1 - p_s(q))}_{A(o^*,\, q)} \nabla_\theta \log \pi_\theta(o \mid q) \right] \Big]
\end{aligned}
\end{equation}


The oracle trajectory $o^*$ serves as a deterministic surrogate for the intractable conditional expectation $\mathbb{E}_{o \sim \pi_\theta \mid R(o)=1}[\cdot]$, translating to: $\nabla_\theta \log \pi_\theta(o \mid q) \approx \nabla_\theta \log \pi_\theta(o^* \mid q)$.

\paragraph{Justification for Asymmetric Substitution.} 
A technical subtlety arises in this transition: the success probability $p_s(q)$ is substituted by $1$ in the empirical frequency slot ($\Pr[R=1]$) but retained within the inner advantage term $(1 - p_s(q))$. We justify this asymmetric treatment through the lens of \textit{distributional decoupling}:
\begin{itemize}
    \item \textbf{Directional Alignment (Data-centric):} SFT performs a strict distributional shift, replacing the model's stochastic generative distribution $\pi_\theta(\cdot \mid R(o)=1)$ with the deterministic expert dataset $\mathcal{D}$. In this expert domain, the empirical frequency of encountering an oracle trajectory is unity ($P_{\mathcal{D}}(R=1) = 1$), which nullifies the negative gradient term and anchors the optimization direction to $o^*$.
    \item \textbf{Magnitude Regulation (Model-centric):} While the optimization \textit{direction} is dictated by the static data distribution, the update \textit{magnitude} must respect the variance-reduced baseline governed by the evolving policy. Retaining $(1 - p_s(q))$ serves as a crucial regularizer to evaluate the competence gap between the current policy state and the oracle target.
\end{itemize}
Consequently, this formulation is not an algebraic identity, but a regularized \textit{Surrogate Gradient} that decouples data-driven direction from policy-driven step-size. Standard SFT applies the distributional shift but simplifies the policy magnitude to a static value of 1, yielding:

\begin{equation}
    \nabla_\theta \mathcal{J}_{\text{SFT}} \propto \mathbb{E}_{(q,o^*) \sim \mathcal{D}} 
    \left[ 1 \cdot 
    \nabla_\theta \log \pi_\theta(o^* \mid q) \right]
\end{equation}

\subsection{OSW-FT Gradient}

\paragraph{Advantage-Guided SFT Magnitude.}
We argue that while the update \textit{frequency} and \textit{direction} should be guided by the oracle $o^*$, the \textit{magnitude} must respect the variance-reduced baseline. Retaining the oracle's deterministic frequency while restoring the advantage magnitude $weight$ yields the \textbf{OSW-FT} gradient:

\begin{equation}
    \nabla_\theta \mathcal{J}_{\text{OSW}} \propto \mathbb{E}_{(q, o^*) \sim \mathcal{D}} \!\left[ {\scriptscriptstyle (1 - \hat{p}(q))} \cdot \nabla_\theta \log \pi_\theta(o^* \mid q) \right]
\end{equation}



where $\hat{p}(q) = \frac{1}{K} \sum_{k=1}^{K} R(\hat{o}_k)$ is an online Monte Carlo estimate of the success probability. Consequently, $(1-\hat{p}(q))$ serves as a principled heuristic that modulates the expert trajectory update using the model's empirical advantage.

\begin{algorithm}[t]
\caption{Online Self-Weighted Fine-Tuning}
\label{alg:osw_ft}
\begin{algorithmic}[1]
\REQUIRE Dataset $\mathcal{D}$, Model $\pi_\theta$, Rollout $K$, Batch Size $B$

\WHILE{not converged}
    \STATE $\mathcal{B} = \{(q_i, o^*_i)\}_{i=1}^{B} \sim \mathcal{D}$

    \STATE Initialize buffer $\mathcal{M} \leftarrow \emptyset$
    \FOR{each $q_i$ in $\mathcal{B}$}
        \STATE Sample $K$ outputs $\{\hat{o}_{i,k}\} \sim \pi_\theta(\cdot|q_i)$
        \STATE Calculate rewards and success probability $\hat{p}_i$
        \STATE Compute weight $w_i = 1 - \hat{p}_i$
        \STATE Store tuple $(q_i, o^*_i, w_i)$ into $\mathcal{M}$
    \ENDFOR
    
    \STATE Compute weighted SFT loss on \textbf{current} $\pi_\theta$:
    $$ \mathcal{L}(\theta) = - \frac{1}{|\mathcal{M}|} \sum_{(q, o^*, w) \in \mathcal{M}} w \cdot \log \pi_\theta(o^* \mid q) $$
    
    \STATE Update parameters.
    
\ENDWHILE
\end{algorithmic}
\end{algorithm}
\vspace{-1em}

\subsection{Discussion}

OSW-FT preserves the original data distribution and dynamically assesses the learning value of each sample online via model rollouts, re-allocating only the gradient \emph{magnitude} through the advantage weight. This ability to adaptively optimize directly on datasets with unlabeled difficulty eliminates the need for manual data sorting. Unlike rollout-optimized RFT methods such as GRPO, the gradient \emph{direction} in OSW-FT is anchored by the high-quality expert demonstration $o^*$; only the magnitude is modulated by the model's current capability. This retains supervised optimization while incorporating an online estimate of the model's current success rate.
The additional online sampling relative to standard SFT is $K$ inference-only rollouts per query. In our experiments, $K=2$ captures a useful advantage signal under a limited online rollout budget. Finally, OSW-FT has a gradient-level structural connection to SFT and RL. Appendix \ref{app:sec:rl} presents this connection under policy-gradient theorem \citep{sutton1999policygradientmethods, schulman2017proximal}.

\section{Theoretical Analysis}
\label{sec:theory}

We establish two key theoretical properties of OSW-FT: robustness to finite rollout estimation and convergence guarantees with smaller gradient second moment than SFT. Full proofs are deferred to Appendix \ref{app:sec:theory}.

\subsection{Robustness to Finite Rollouts}

Success probability $p_s(q)$ is estimated via $K$ rollouts: $\hat{p}(q) = \frac{1}{K}\sum_{k=1}^{K} R(\hat{o}_k)$. 
Denoting $g_0(q) := \nabla_\theta \log \pi_\theta(o^*|q)$, we analyze 
the OSW-FT estimator $\hat{g}_K(q) = (1 - \hat{p}(q)) \cdot g_0(q)$.

\begin{theorem}[Finite Rollout Properties]\label{thm:finite_K_main}
For a fixed query $q$ with $p := p_s(q) \in (0,1)$ and $K$ i.i.d.\ rollouts, the OSW-FT estimator is unbiased: $\mathbb{E}_{\hat{p}}[\hat{g}_K(q)] = (1-p) \cdot g_0(q)$.
\end{theorem}
 
\begin{remark}[Choice of $K$]\label{rem:K}
The signal loss probability $p^K$ predicts a qualitative transition at $K = 2$. At $K = 1$, the weight is binary ($\{0, 1\}$) and a query with $p = 0.5$ loses its signal with 50\% probability, reducing OSW-FT to hard example selection. At $K = 2$, signal loss drops to $p^2 = 25\%$ and the intermediate weight $w = 0.5$ enables soft discrimination. For $K \geq 2$, $p^K$ is already small and further increases yield diminishing returns. 

\end{remark}


\subsection{Convergence Analysis}
 
A subtlety of OSW-FT is that the weight $1 - p_s(q)$ depends on the current policy $\pi_\theta$, yet Algorithm~\ref{alg:osw_ft} holds it fixed within each batch of gradient updates. To handle this, we define a surrogate objective $J^{(t)}(\theta)$ at each outer iteration $t$, which freezes the weights at their current values while optimizing only the log-likelihood (see Definition~\ref{def:surrogate} in Appendix). Under standard smoothness and bounded-gradient assumptions (Assumption~\ref{assm:regularity} in Appendix), the following guarantee holds.
 
\begin{theorem}[Convergence Rate]\label{thm:convergence_main}
Suppose OSW-FT runs for $T$ total gradient steps with learning rate $\eta = 1/\sqrt{T}$, updating weights every $S$ steps. Let $L$ be the smoothness constant of $J^{(t)}$, $B$ an upper bound on $|\log \pi_\theta(o^*|q)|$, and $\epsilon_p$ the maximum change in $p_s(q)$ between consecutive weight updates. Then:


\begin{equation}
\begin{aligned}
&\frac{1}{T}\sum_{t=1}^{T} \mathbb{E}\!\left[\|\nabla_\theta J^{(\lfloor t/S \rfloor)}(\theta_t)\|^2\right] \\
&\leq \frac{2\Delta_0 + L \cdot \Sigma^2}{\sqrt{T}} + B \cdot \epsilon_p,
\end{aligned}
\end{equation}

where $\Delta_0$ is the initial optimality gap. The stochastic-gradient second moment is


\begin{equation}
\begin{aligned}
    \Sigma^2 &:= \mathbb{E}_{q \sim D} \Big[ \left( (1-p_s(q))^2 \right. \\
    &\quad \left. +\, \frac{p_s(q)(1-p_s(q))}{K} \right) \|g_0(q)\|^2 \Big],
\end{aligned}
\end{equation}

which satisfies $\Sigma^2 < \Sigma^2_{\mathrm{SFT}} := \mathbb{E}_{q \sim D}[\|g_0(q)\|^2]$ for any $K \geq 1$, whenever the model has nonzero success probability on a positive-measure set of queries.
\end{theorem}
 
The first term vanishes at rate $O(1/\sqrt{T})$; the second term $B \cdot \epsilon_p$ reflects the cost of holding weights fixed between updates, which is small under typical learning rates. The strict inequality $\Sigma^2 < \Sigma^2_{\mathrm{SFT}}$ implies that OSW-FT converges with smaller second moment than SFT under identical conditions.

\section{Experiments}
\label{sec:exp}

\begin{table*}[t]
    \centering
    \caption{\textbf{Main Results.} Comparison across different post-training paradigms. We report \textbf{Pass@1 / Pass@16} for competition benchmarks (AIME, AMC) to highlight both stability and exploration potential. For MATH-500, GPQA-Diamond, we report \textbf{Pass@1}. \textbf{Bold} indicates the best performance. $^\dagger$ indicates the base model shows lower performance due to limited instruction-following ability, which recovers after fine-tuning.}
    \vspace{-1em}
    \label{tab:main_results}
    \resizebox{\textwidth}{!}{%
    \begin{tabular}{lccccccc}
    \toprule
    \textbf{Model} & \textbf{AIME 24} & \textbf{AIME 25} & \textbf{AMC 22} & \textbf{AMC 23} & \textbf{AMC 24} & \textbf{MATH-500} & \textbf{GPQA-D} \\
    \midrule
    Qwen3-0.6B-Base & \textbf{00.83} / 06.67 & 00.21 / 03.34 & \textbf{11.05} / \textbf{41.86} & 11.82 / 54.35 & 04.58 / 26.67 & 22.00 & 14.14 \\
    \midrule
    \quad + GRPO & \textbf{00.83} / \textbf{10.00} & 00.00 / 00.00 & 08.14 / 25.58 & \textbf{16.30} / 43.48 & \textbf{08.61} / \textbf{31.11} & \textbf{36.20} & \textbf{26.77} \\
    \quad + SFT & 00.00 / 00.00 & 00.21 / 03.34 & 09.01 / 32.56 & 11.55 / 50.00 & 03.89 / 24.44 & 24.20 & 15.66 \\
    \rowcolor{lightgray}
    \quad + OSW-FT & \textbf{00.83} / 06.67 & \textbf{01.04} / \textbf{10.00} & 08.72 / \textbf{41.86} & 13.04 / \textbf{56.52} & 05.14 / 28.89 & 24.20 & 20.20 \\
    \midrule
    Qwen3-1.7B-Base & 03.75 / 20.00 & \textbf{03.96} / 20.00 & \textbf{23.84} / \textbf{58.14} & 28.67 / \textbf{65.22} & 14.17 / \textbf{48.89} & 56.40 & 25.76 \\
    \midrule
    \quad + GRPO & \textbf{06.25} / 23.33 & 02.08 / 13.33 & 20.20 / 48.84 & \textbf{34.10} / \textbf{65.22} & 13.06 / 42.22 & \textbf{59.40} & \textbf{28.79} \\
    \quad + SFT & 02.92 / 20.00 & 03.12 / 10.00 & 18.31 / 55.81 & 26.90 / \textbf{65.22} & 13.75 / \textbf{48.89} & 52.20 & 17.17 \\
    \rowcolor{lightgray}
    \quad + OSW-FT & 05.42 / \textbf{26.67} & \textbf{03.96} / \textbf{23.34} & 22.24 / 55.81 & 27.45 / \textbf{65.22} & \textbf{17.78} / \textbf{48.89} & 57.40 & 26.27 \\
    \midrule
    Qwen3-4B-Base & 09.58 / 30.00 & 06.88 / 23.34 & 36.63 / 69.77 & 39.27 / \textbf{78.26} & 29.31 / 60.00 & 69.60 & 10.10 $^\dagger$ \\
    \midrule
    \quad + GRPO & 05.63 / 20.00 & 05.20 / 23.33 & 39.53 / 58.14 & 34.92 / 67.39 & 22.92 / 53.33 & 71.40 & 30.32 \\
    \quad + SFT & 12.29 / 26.67 & 15.42 / 26.67 & 40.70 / 69.77 & 44.57 / 73.91 & \textbf{35.00} / \textbf{75.56} & 76.40 & 27.27 \\
    \rowcolor{lightgray}
    \quad + OSW-FT & \textbf{13.33} / \textbf{36.67} & \textbf{16.25} / \textbf{36.67} & \textbf{41.42} / \textbf{74.42} & \textbf{45.11} / \textbf{78.26} & \textbf{35.00} / \textbf{75.56} & \textbf{79.00} & \textbf{32.32} \\
    \bottomrule
    \end{tabular}%
    }
    \vspace{-1em}
\end{table*}

\subsection{Experimental Setup}
We evaluate our method across model scales with the Qwen3 series (0.6B, 1.7B, and 4B) \citep{yang2025qwen3technicalreport} to comprehensively study the interaction between model capacity and our dynamic weighting mechanism. We compare OSW-FT against standard SFT and GRPO \citep{shao2024deepseekmath}. 

\paragraph{Training Dataset.} 
To establish a rigorous testbed for reasoning capabilities, we curated a high-quality training set of 10k multiple-choice questions (MCQs) sampled from AM-Qwen3-Distilled \citep{tian2025correctanswersequaldistillation} --- a NuminaMATH dataset \citep{numina_math_datasets} distilled via Qwen3-235B.
To ensure the provision of high-quality oracle trajectories ($o^*$) required for both SFT and OSW-FT, we manually filter and clean the dataset. Specifically, we rigorously enforce a single-choice answer format and constrain the reasoning trajectories to a maximum length of 4k tokens with the vast majority around 2k tokens. This length distribution ensures that even small-capacity base models can learn from substantial reasoning steps efficiently.

\paragraph{Experiment Details.}
All models undergo full-parameter fine-tuning. We employ the AdamW optimizer \citep{loshchilov2017decoupledadamw} with a constant learning rate schedule, alongside \texttt{DeepSpeed ZeRO3} and \texttt{bfloat16} precision across all model scales. 
To minimize generation overhead, we integrate \texttt{vLLM} \citep{kwon2023efficientvllm} for both training rollouts and evaluation. Its high-throughput capabilities enable us to maintain a large training rollout batch size (128), making the online cost of OSW-FT highly efficient.

The maximum sequence length is set to 5k tokens, with a maximum completion length of 4k tokens allocated for online rollouts. For OSW-FT and GRPO, we use Math-Verify \citep{huggingface2025mathverify} for binary accuracy reward function. 

\begin{itemize}
    \item \textbf{SFT:} We fine-tune the models for 1 epoch with a learning rate of $1 \times 10^{-5}$.
    \item \textbf{OSW-FT (Ours):} Following the same 1-epoch schedule and $1 \times 10^{-5}$ learning rate. For each query, we generate $K=2$ trajectories at a temperature of $T=1.0$ to estimate the online success probability $\hat{p}(q)$.
    \item \textbf{GRPO:} For each query, we generate $K=8$ trajectories (group size of $G=8$) at a temperature of $T=1.0$. We adopt $1 \times 10^{-6}$ learning rate, aligning with established community best practices for stable RL training. To mitigate variance from trajectory lengths, we employ token-level policy gradient loss (DAPO, \citep{yu2025dapo}). We set KL penalty coefficient $\beta=0$ (disabling KL) and apply clipping bound of $[0.8, 1.28]$.
\end{itemize}

\paragraph{Evaluation Protocol.}
To thoroughly assess verifiable reasoning capabilities, we report performance on several challenging benchmarks, including advanced mathematical reasoning tasks such as AMC \citep{amc2023misc}, AIME \citep{aime2025misc}, and MATH-500 \citep{lightman2023let}, as well as the expert-level general question answering GPQA-Diamond \citep{rein2024gpqa}.
We report \textbf{Pass@1}/\textbf{Pass@16} for competition benchmarks (AMC, AIME). We use Pass@16 as an empirical indicator of exploration: under the same 16-sample budget, a higher Pass@16 means that the model can discover a correct solution on more queries through sampling. Additionally, for benchmarks comprising hundreds of items (MATH-500 and GPQA-Diamond), we report \textbf{Pass@1} as an auxiliary test to evaluate deterministic problem-solving accuracy on relatively simpler or out-of-distribution (OOD) tasks.
We apply sampling with $Top_p$=0.95, $Top_k=20$, and a sampling temperature of $T=0.6$, allowing a maximum generation length of 8k tokens.

\subsection{Main Results}
Table \ref{tab:main_results} presents the main comparative results. Overall, OSW-FT improves over standard SFT across the evaluated small-to-medium model scales. Its comparison with GRPO varies across models and metrics, reflecting a trade-off rather than a general performance ordering between expert-anchored training and RFT.

\paragraph{Enhancing Exploration in Small Language Models.} 
For the 0.6B, 1.7B, and 4B models, OSW-FT consistently outperforms or matches standard SFT on Pass@16. For instance, on the 4B model, OSW-FT achieves a Pass@16 of 36.67\% on both AIME 24 and AIME 25, compared with 26.67\% for SFT. From these Pass@16 results, we observe that OSW-FT encourages more exploration under the same sampling budget: it discovers correct solutions for a larger fraction of queries across repeated samples. Relative to GRPO, OSW-FT performs favorably on several competition-math metrics, while GRPO remains stronger in some other settings, particularly on Pass@1 for the smaller models.

\paragraph{The Capacity-Data Mismatch Phenomenon.}
The optimization dynamics of OSW-FT further explain its scale-dependent behavior. As illustrated in Figure \ref{fig:osw_weight_dynamics}, the advantage weight $(1-\hat{p})$ acts as an implicit curriculum, decaying as the model masters the dataset. Crucially, these trajectories cleanly stratify by model capacity: the highly capable 4B model masters the dataset rapidly, causing its weights to plummet toward zero, whereas the 0.6B model maintains a sustained learning signal. 


\begin{figure}
    \centering
    \includegraphics[width=\linewidth]{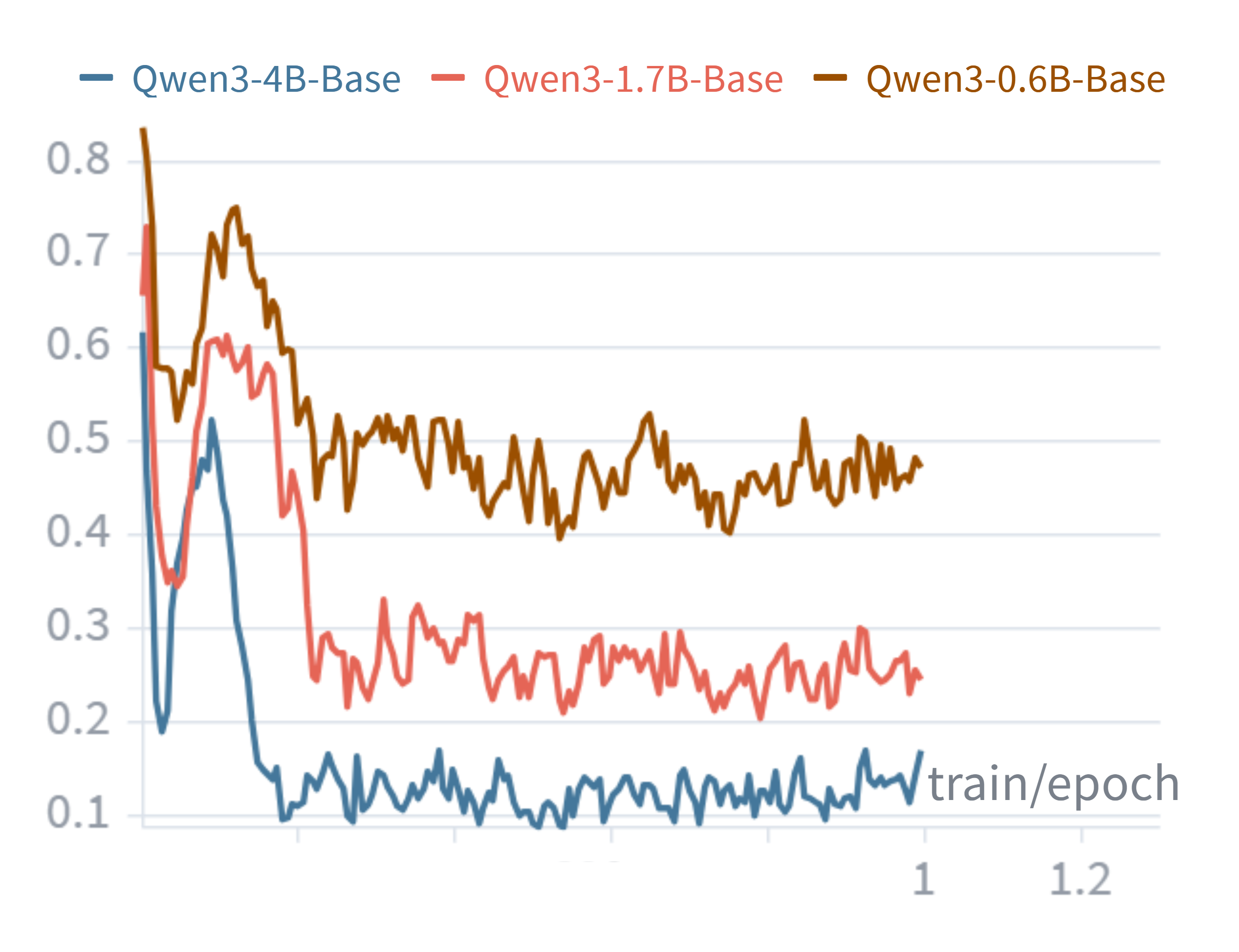}
    \vspace{-2em}
    \caption{OSW-FT weight $(1-\hat{p})$ stratifies by model scale, validating the model-data matching hypothesis.}
    \label{fig:osw_weight_dynamics}
    \vspace{-2em}
\end{figure}

\begin{table*}[t]
    \centering
    \caption{\textbf{Ablation on the number of rollouts ($K$).} We report \textbf{Pass@16} for AIME 24, AIME 25 and \textbf{Pass@1} for MATH-500, GPQA-D on Qwen3-0.6B-Base. OSW-FT maintains robust performance at $K=2$, while GRPO is more sensitive to the rollout count on the AIME benchmarks in this setting.}
    \vspace{-1em}
    \label{tab:ablation_k}
    \resizebox{0.8\textwidth}{!}{%
    \begin{tabular}{@{} l cccc cccc @{}}
    \toprule
    \multirow{2}{*}{\textbf{Rollouts}} & \multicolumn{4}{c}{\textbf{GRPO}} & \multicolumn{4}{c}{\textbf{OSW-FT}} \\
    \cmidrule(lr){2-5} \cmidrule(lr){6-9}
    & \textbf{AIME 24} & \textbf{AIME 25} & \textbf{MATH-500} & \textbf{GPQA-D} & \textbf{AIME 24} & \textbf{AIME 25} & \textbf{MATH-500} & \textbf{GPQA-D} \\
    \midrule
    $K = 1$ & -- & -- & -- & -- & \textbf{06.67} & 06.67 & 22.80 & 15.15 \\
    $K = 2$ & 00.00 & 00.00 & 29.00 & 22.72 & \textbf{06.67} & \textbf{10.00} & 24.20 & 20.20 \\
    $K = 4$ & 06.67 & 00.00 & \textbf{37.00} & 26.27 & \textbf{06.67} & \textbf{10.00} & 24.00 & \textbf{20.71} \\
    $K = 8$ & \textbf{10.00} & 00.00 & 36.20 & \textbf{26.77} & \textbf{06.67} & \textbf{10.00} & \textbf{24.60} & \textbf{20.71} \\
    \bottomrule
    \end{tabular}%
    }
    \vspace{-1em}
\end{table*}

\subsection{Performance under Limited Rollout Budgets}
As established in Remark \ref{rem:K} (Section \ref{sec:theory}), our theoretical analysis of the signal loss probability predicts a critical qualitative transition at $K=2$. At this minimal threshold, the advantage weight shifts from binary hard-example mining to fine-grained soft discrimination. Guided by this theoretical insight, a central computational consideration is empirically verifying this optimal trade-off against the overhead introduced by the online rollouts. Because these rollouts strictly serve to estimate the baseline for variance reduction rather than to discover the correct reasoning path from scratch, minimal generation overhead should theoretically suffice. To validate this, we conduct an ablation study on the Qwen3-0.6B-Base model, varying $K \in \{1, 2, 4, 8\}$. The results are shown in Table \ref{tab:ablation_k}.

\begin{table}[t]
    \centering
    \caption{\textbf{Ablation on alternative weighting strategies ($K = 4$).} We report \textbf{Pass@16} for AIME 24, AIME 25 and \textbf{Pass@1} for MATH-500, GPQA-D on Qwen3-0.6B-Base. OSW-FT consistently outperforms alternative weighting heuristics, validating the necessity of smooth online success-rate anchoring.}
    \vspace{-1em}
    \label{tab:ablation_w}
    \resizebox{0.48\textwidth}{!}{%
    \begin{tabular}{@{} l cccc @{}}
    \toprule
    \textbf{Strategy} & \textbf{AIME 24} & \textbf{AIME 25} & \textbf{MATH-500} & \textbf{GPQA-D} \\
    \midrule
    Hard-thresholded & \textbf{06.67} & 06.67 & 23.80 & 15.15 \\
    Random & 00.00 & 00.00 & 21.40 & 15.17 \\
    \midrule
    \textbf{OSW-FT} & \textbf{06.67} & \textbf{10.00} & \textbf{24.00} & \textbf{20.71} \\
    \bottomrule
    \end{tabular}%
    }
    \vspace{-1em}
\end{table}

\paragraph{OSW-FT.}
As shown in Table \ref{tab:ablation_k}, setting $K=1$ acts as a highly noisy, binarized estimator and yields the lowest performance across all metrics (e.g., $15.15\%$ on GPQA-D). However, simply increasing the rollouts to $K=2$ produces a sharp performance improvement, jumping to $20.20\%$ on GPQA-D and matching or exceeding $K=1$ on all other benchmarks. Critically, the performance plateaus for $K > 2$. Increasing the rollouts to $K=4$ or $K=8$ yields only negligible marginal benefits. Consequently, a coarse estimate via $K=2$ captures the necessary advantage signal without demanding excessive generation overhead.

\paragraph{Contrast with GRPO under Limited Rollout Budgets.}
To fully contextualize this computational efficiency, we contrast OSW-FT with the GRPO baseline under matched rollout budgets. Table \ref{tab:ablation_k} suggests a potential limitation of pure RL: while GRPO can optimize shorter-horizon tasks (MATH, GPQA) at minimal compute ($K=2$), it obtains zero scores on the most challenging reasoning benchmarks, scoring \textbf{0.00} on both AIME 24 and AIME 25. This outcome may stem from RL's reliance on model-generated exploration. On highly complex tasks like AIME, the probability of finding a correct long-horizon reasoning path with only 2 random attempts can be low. Consequently, the advantage signal for these hard queries can become sparse, providing the model with limited learning signal. To alleviate this exploration bottleneck and achieve non-zero AIME performance, GRPO can benefit from larger group sizes (e.g., $K=8$).

\paragraph{Advantage of GT Anchoring.}
This contrast suggests a potential advantage of OSW-FT for capacity-constrained models. By deterministically anchoring the gradient direction to the ground-truth oracle trajectory ($o^*$), OSW-FT provides a stable, expert-anchored optimization path across the evaluated tasks. Thus, it can mitigate the exploration bottleneck that can affect pure RL on hard queries, showing consistent performance across the evaluated benchmarks while operating at a limited computational budget ($K=2$).

\begin{figure*}[t]
    \centering
    \begin{subfigure}[b]{0.49\linewidth}
        \centering
        \includegraphics[width=\linewidth]{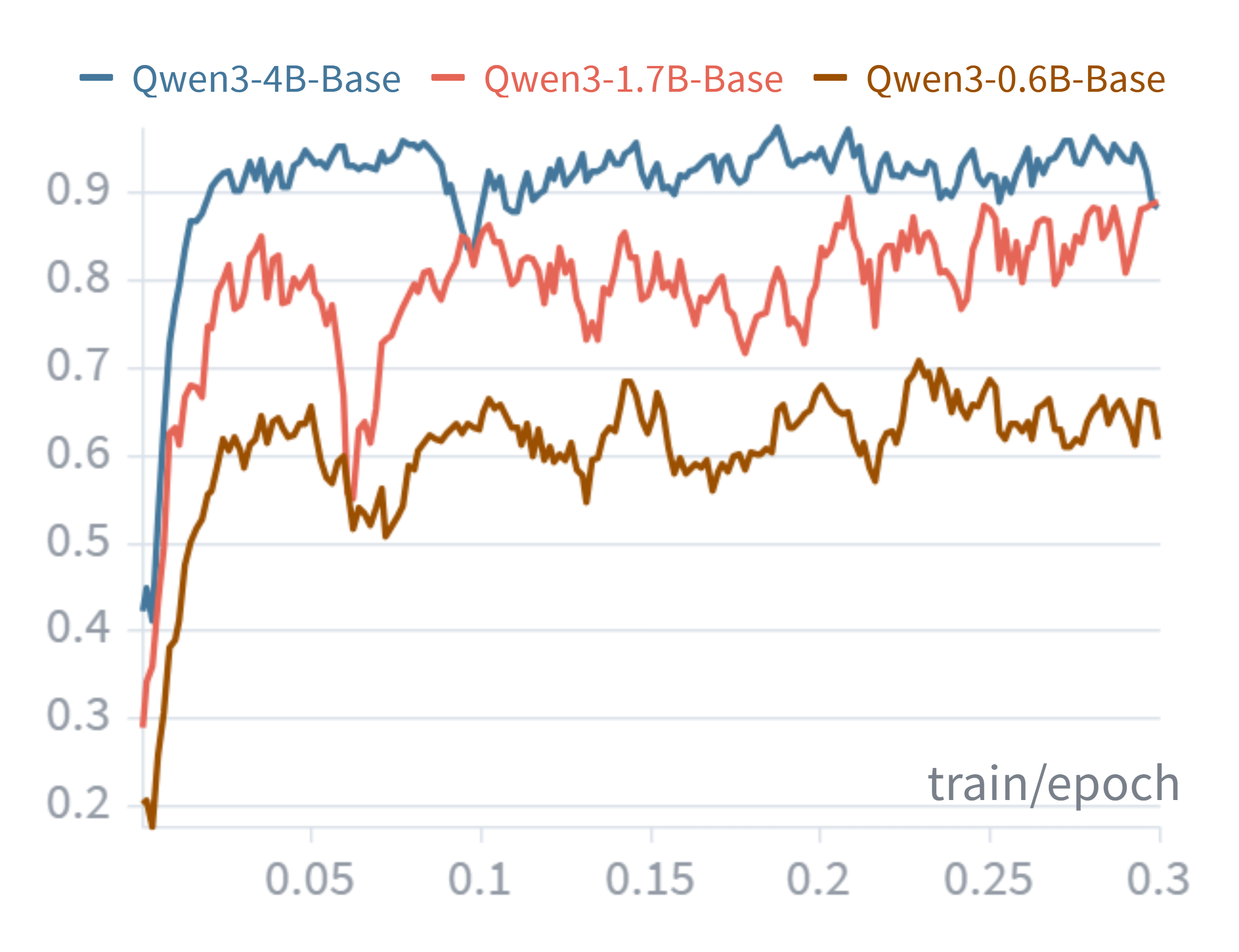}
        \vspace{-2em}
        \caption{GRPO learning rate of $1 \times 10^{-6}$}
        \label{fig:grpo_reward_dynamics}
    \end{subfigure}
    \hfill
    \begin{subfigure}[b]{0.49\linewidth}
        \centering
        \includegraphics[width=\linewidth]{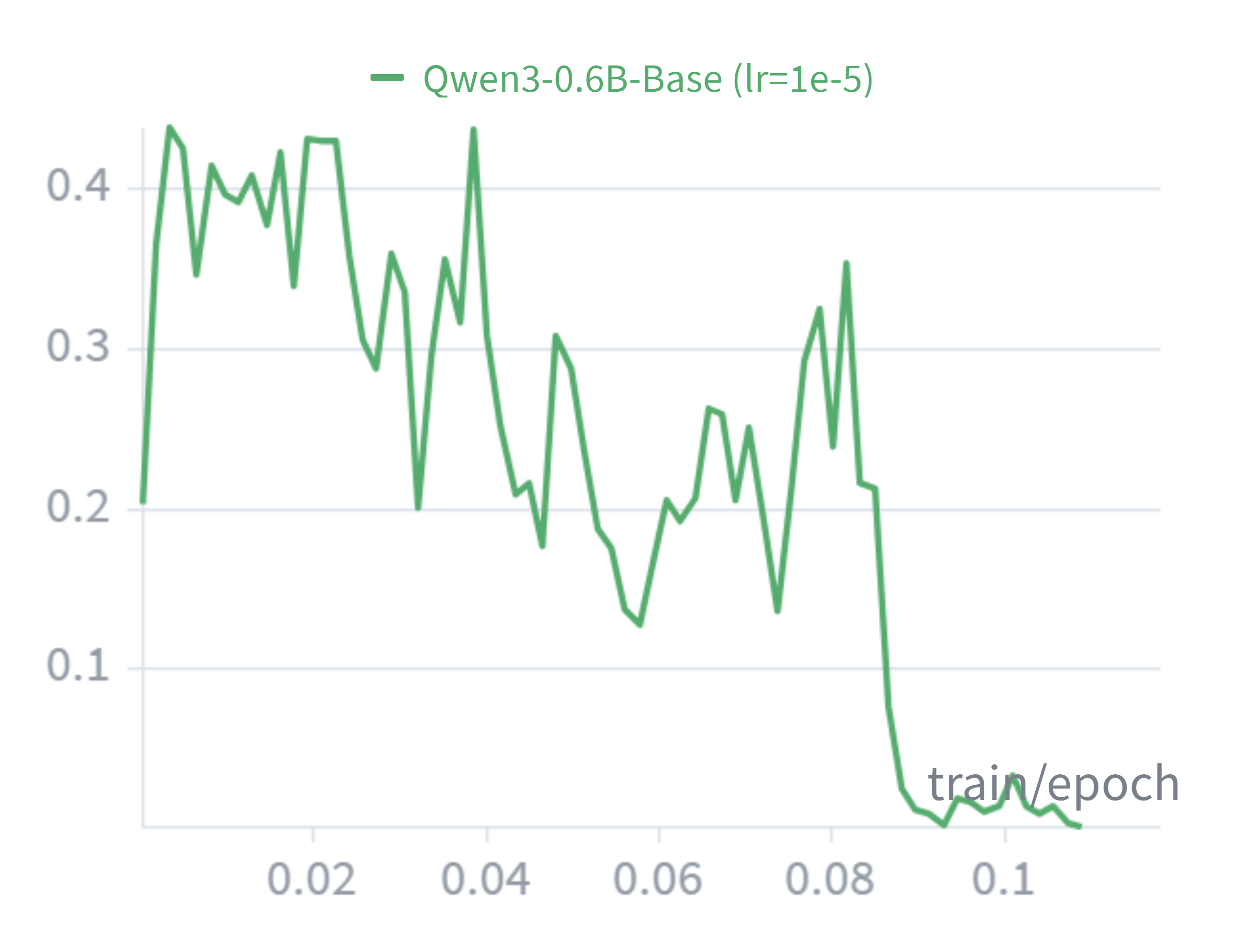}
        \vspace{-2em}
        \caption{GRPO learning rate of $1 \times 10^{-5}$}
        \label{fig:grpo_lr}
    \end{subfigure}
    \caption{GRPO training reward under two learning rates. In our setup, the reward collapses toward zero at $1 \times 10^{-5}$ across the evaluated model scales, whereas $1 \times 10^{-6}$ produces the stable runs used in the main comparison.}
    \label{fig:grpo_reward_lr}
    \vspace{-1em}
\end{figure*}

\subsection{Ablation on Weighting Strategies}

To thoroughly investigate whether the performance gains of OSW-FT stem from the mathematically derived online success-rate weighting rather than naive data filtering or generic optimization artifacts, we conduct a rigorous ablation study on Qwen3-0.6B-Base under rollout budget of $K=4$. We implement two critical control baselines:

\begin{itemize}
    \item \textbf{Hard-thresholded Weighting:} A discrete control scheme where a query receives a binary weight $w \in \{0, 1\}$. Specifically, $w = 0$ if all $K$ rollouts succeed (indicating the query is fully mastered), and $w = 1$ otherwise. This mimics hard rejection sampling or hard mistake-driven learning.
    \item \textbf{Random Weighting:} A control scheme that decouples the loss weights from the model's actual performance. It assigns random weights drawn from a uniform distribution, calibrated to match the exact empirical mean weight of OSW-FT across the dataset, thereby testing if our method simply acts as an implicit learning rate decay.
\end{itemize}

The empirical results are summarized in Table~\ref{tab:ablation_w}. We observe a severe performance degradation when utilizing alternative heuristics. In contrast, OSW-FT scales down gradient updates as capability increases, supporting the benefit of content-aware weighting under a limited online rollout budget.

\subsection{On the Optimization Stability of GRPO}

\paragraph{Learning Rate Sensitivity.}
We initially trained the GRPO baseline using the $1 \times 10^{-5}$ learning rate employed by SFT and OSW-FT. In our setup, this learning rate causes the training reward to collapse toward zero across the evaluated model scales, as shown in Figure \ref{fig:grpo_lr}. Once all sampled trajectories receive the same reward, the relative-advantage signal vanishes and optimization stalls.

We therefore use $1 \times 10^{-6}$ for the reported GRPO baseline. This lower learning rate produces the stable training dynamics shown in Figure \ref{fig:grpo_reward_dynamics}.

\paragraph{Late-Stage Instability and Checkpoint Selection.}
We also observe late-stage degradation in some GRPO runs, particularly for the smaller-capacity models. To report a stable GRPO baseline, we evaluate the checkpoint at which the reward reaches a plateau, before subsequent degradation, rather than using the final checkpoint unconditionally.

\section{Related Work}
\label{sec:related}

\paragraph{Offline Reweighting and Regularization.} 
A growing body of work views standard SFT through an offline Reinforcement Learning lens, aiming to mitigate over-optimization without incurring the computational cost of online generation. One line of research focuses on reweighting (e.g., SoftDedup \citep{he2024softdedup}, iw-SFT \citep{qin2025supervisediwsft}, DFT \citep{wu2026ondft}).
Another line introduces structural regularization, such as ASFT \citep{zhu2026anchored}, which apply KL-divergence penalties against a static reference model to prevent catastrophic forgetting and stylistic overfitting.
While these approaches improve upon uniform SFT, they operate entirely offline. OSW-FT bypasses these static heuristics by operating online, directly tracking the model's continuously evolving capabilities via real-time rollouts.

\paragraph{Online Rollout-Driven Alignment.} 
To incorporate real-time generative feedback, recent paradigms increasingly rely on online rollouts \citep{schulman2017proximal, shao2024deepseekmath, yu2025dapo, zheng2025groupsequencepolicyoptimization, gao2025softsapo, wang2025aspo, chen2025minimax}. STaR \citep{zelikman2022star} filters and trains exclusively on self-generated rationales that yield correct final answers. RL or hybrid-RL methods like GRPO \citep{shao2024deepseekmath} and CHORD \citep{zhang2026onpolicy} are inherently bottlenecked by massive compute overhead.
Conversely, rollout-augmented SFT methods attempt to simplify this optimization but introduce new flaws. OSFT \citep{li2025onlinesftllmreasoning} fine-tunes on model-generated trajectories without filtering by correctness, and lacks explicit mechanisms for accuracy-based gradient scaling. OTR \citep{ming2025one} guides SFT using a localized policy gradient. However, its one-token lookahead provides only myopic, token-level credit signals, lacking explicit trajectory-level credit assignment for long-horizon reasoning.
OSW-FT resolves these bottlenecks by unifying trajectory evaluation with deterministic optimization.

\paragraph{Data Selection and Curriculum Learning.} 
Recent findings highlight that alignment depends heavily on data quality rather than sheer quantity \citep{zhou2023lima, li2024quantity}. To maximize data efficiency, traditional curriculum learning \citep{Bengio2009CurriculumL} presents training examples in a manually predefined order of increasing difficulty. This static paradigm has largely evolved into Automatic Curriculum Learning \citep{narvekar2020curriculum, portelas2020automatic}, where intermediate tasks are dynamically selected based on the agent's real-time competence. 
More recently, curriculum paradigms have been adapted for LLMs \citep{shi2025efficient}, alongside advanced data selection methods like FisherSFT \citep{deb2025fishersft}. However, these LLM-specific approaches typically rely on static heuristics that struggle to accurately track the model's continuously evolving capabilities during training. 
OSW-FT seamlessly shifts the optimization budget toward the model's true capability frontier, entirely eliminating the need for offline filtering or manual data sorting.


\section{Conclusion}
\label{sec:conclusion}

We introduced \textbf{Online Self-Weighted Fine-Tuning (OSW-FT)}, a simple method that augments supervised fine-tuning with online, trajectory-level weighting for verifiable reasoning. By estimating the model's current success rate with a small number of inference-only rollouts, OSW-FT adjusts how strongly each expert trajectory should update the model while keeping the optimization direction anchored to expert demonstrations. Empirically, OSW-FT consistently improves over SFT on small-to-medium models and offers a favorable performance--rollout trade-off when the available online rollout budget is limited.

\section*{Limitations}
\label{sec:limitations}

OSW-FT currently assumes a reliable binary verifier and high-quality expert trajectories. Its success-rate weight is designed for binary-verifiable reasoning and has not been validated for continuous rewards or open-ended generation. Although GPQA-Diamond broadens the evaluation beyond mathematics, it remains a binary-verifiable multiple-choice task and therefore does not establish applicability to non-binary settings.

Our experiments also focus on Qwen3 model family and primarily mathematical reasoning tasks. Extending OSW-FT to code generation and other reward formulations remains future work.



\bibliography{main}

\newpage
\appendix
\section{Ethics Statement}
\label{app:sec:ethics}

We acknowledge the use of Large Language Models (LLMs) during the preparation of this manuscript. Specifically, LLMs were utilized solely as auxiliary tools for minor writing assistance (such as grammar checking, typo correction, and phrasing refinement) and for debugging code implementations. We have reviewed all LLM-assisted output and take full responsibility for the accuracy, originality, and integrity of the work presented in this paper.

\section{Theoretical Proofs}
\label{app:sec:theory}

Throughout, we use $p := p_s(q) = \mathbb{E}_{o \sim \pi_\theta}
[R(o)|q]$ and $g_0(q) := \nabla_\theta \log \pi_\theta(o^*|q)$ 
for a fixed query $q$ satisfying Assumption~\ref{assm:unique_response}.

\begin{assumption}[Binary Verifiable Reward]\label{assm:unique_response}
For each query $q$, the reward is binary: $R(o) \in \{0, 1\}$, and the 
success probability $p_s(q) := \mathbb{E}_{o \sim \pi_\theta}[R(o)|q] 
\in (0,1)$ is well-defined.
\end{assumption}
 
\subsection{Finite Rollout Analysis (Proof of Theorem~\ref{thm:finite_K_main})}
\label{app:subsec:finite_k}
 
\begin{proof}
Let $\hat{p}(q) = \frac{1}{K}\sum_{k=1}^K R(\hat{o}_k)$ where $\hat{o}_k \sim \pi_\theta(\cdot|q)$ are i.i.d.\ rollouts. Since $R(\hat{o}_k) \sim \mathrm{Bernoulli}(p)$, we have $\mathbb{E}[\hat{p}] = p$ and $\mathrm{Var}[\hat{p}] = p(1-p)/K$.

\noindent\textbf{Unbiasedness.}
Since $g_0(q)$ is deterministic for a fixed query $q$:
\begin{equation}
\mathbb{E}_{\hat{p}}[\hat{g}_K(q)] = \mathbb{E}[1 - \hat{p}(q)] \cdot g_0(q) = (1 - p) \cdot g_0(q).
\end{equation}
 
\noindent\textbf{Signal loss.}
The event $w(q) = 0$ occurs iff $\hat{p}(q) = 1$, i.e., all $K$ rollouts are correct. By independence:
\begin{equation}
P(w(q) = 0) = P(\hat{p}(q) = 1) = p^K.
\end{equation}
\end{proof}
 
The following results provide additional quantitative characterization of the estimator under finite rollouts.
 
\begin{remark}[Total second moment]\label{rem:second_moment}
Using the identity $\mathbb{E}[X^2] = (\mathbb{E}[X])^2 + \mathrm{Var}[X]$ with $X = 1 - \hat{p}$:
\begin{equation}
\mathbb{E}_{\hat{p}}[\|\hat{g}_K(q)\|^2] = \left[(1-p)^2 + \frac{p(1-p)}{K}\right]\|g_0(q)\|^2.
\end{equation}
This decomposes into a signal term $(1-p)^2\|g_0\|^2$ and an estimation noise term $\frac{p(1-p)}{K}\|g_0\|^2$ that vanishes as $K \to \infty$. This expression is used in the convergence analysis (Theorem~\ref{thm:convergence_main}).
\end{remark}
 
\begin{remark}[Estimation-induced MSE]\label{rem:mse}
The MSE due to weight estimation can be isolated as:


\begin{equation}
\begin{aligned}
\mathbb{E}_{\hat{p}}\!\left[\|\hat{g}_K(q) - (1-p)g_0(q)\|^2\right] = \\ \frac{p(1-p)}{K}\|g_0(q)\|^2,
\end{aligned}
\end{equation}

which follows from $\hat{g}_K - (1-p)g_0 = (p - \hat{p})g_0$ and $\mathbb{E}[(p-\hat{p})^2] = \mathrm{Var}[\hat{p}] = p(1-p)/K$.
\end{remark}

\subsection{Convergence Analysis (Proof of Theorem~\ref{thm:convergence_main})}
\label{app:subsec:convergence}
 
\begin{definition}[OSW-FT Surrogate Objective]\label{def:surrogate}
At outer iteration $t$ with parameters $\theta_t$, define:
\begin{equation}
J^{(t)}(\theta) = \mathbb{E}_{q \sim \mathcal{D}}\!\left[\underbrace{(1 - p_{\theta_t}(q))}_{\text{fixed weight}} \cdot \log \pi_\theta(o^* \mid q)\right],
\end{equation}
where the weight is computed from $\pi_{\theta_t}$ and held fixed during inner-loop updates.
\end{definition}
 
\begin{remark}[Necessity of stop-gradient]\label{rem:stop_grad}
Differentiating the naive objective $\mathbb{E}_q[(1-p_\theta(q))\log\pi_\theta(o^*|q)]$ through the weight would introduce a $\nabla_\theta p_\theta(q)$ term that the OSW-FT estimator $(1-\hat{p})\,g_0$ does not capture. The surrogate~$J^{(t)}$ makes explicit that $(1-\hat{p})\,g_0$ is an unbiased gradient of $J^{(t)}(\theta)$ at $\theta = \theta_t$, which is the quantity optimized in practice. This matches Algorithm~\ref{alg:osw_ft}, where rollouts and weights are computed once per batch before gradient updates.
\end{remark}
 
\begin{assumption}[Regularity Conditions]\label{assm:regularity}
\begin{enumerate}[label=(\roman*), leftmargin=2em]
    \item \textbf{$L$-smoothness}: For each fixed weight configuration, $J^{(t)}(\theta)$ is $L$-smooth in $\theta$.
    \item \textbf{Bounded second moment}: $\mathbb{E}_{q,\hat{p}}[\|(1-\hat{p}(q))\, g_0(q)\|^2] \leq M$ for all $\theta$.
    \item \textbf{Weight stability}: $|p_{\theta_{t+1}}(q) - p_{\theta_t}(q)| \leq \epsilon_p$ for all $q$ between consecutive outer iterations.
    \item \textbf{Bounded log-likelihood}: $|\log \pi_\theta(o^*|q)| \leq B$ for all $\theta, q$.
\end{enumerate}
\end{assumption}
 
\begin{proof}[Proof of Theorem~\ref{thm:convergence_main}]
The proof proceeds in four steps.
 
\vspace{0.5em}
\noindent\textbf{Step 1: Inner-loop convergence.}
Within each outer iteration $t$, the weights are fixed, and by the unbiasedness established in Theorem~\ref{thm:finite_K_main}, the stochastic gradient $(1-\hat{p}(q))\,g_0(q)$ is an unbiased estimator of $\nabla_\theta J^{(t)}(\theta_t)$. Applying the standard SGD convergence result for $L$-smooth non-convex objectives \citep{ghadimi2013stochastic} over the $S$ inner steps:


\begin{equation}\label{eq:inner_convergence}
\begin{aligned}
&\frac{1}{S}\sum_{s=1}^{S}\mathbb{E}[\|\nabla J^{(t)}(\theta_{t,s})\|^2] \\ 
&\leq \frac{2(J^{(t)}(\theta^*) - J^{(t)}(\theta_{t,0}))}{S\eta} + L\eta\,\Sigma^2,
\end{aligned}
\end{equation}

where


\begin{equation}\label{eq:sigma_def}
\begin{aligned}
&\Sigma^2 := \mathbb{E}_{q \sim \mathcal{D}} \left[ 
\vphantom{\frac{p(q)(1-p(q))}{K}} \right.
(1-p(q))^2 \\
&\left. +\, \frac{p(q)(1-p(q))}{K} \right] \|g_0(q)\|^2
\end{aligned}
\end{equation}

follows from the total second moment expression in Remark~\ref{rem:second_moment} and linearity of expectation over $q$.
 
\vspace{0.5em}
\noindent\textbf{Step 2: Cross-iteration mismatch.}
When switching from outer iteration $t$ to $t+1$, the surrogate objective changes. The mismatch is bounded by:


\begin{equation}
\begin{aligned}
&|J^{(t+1)}(\theta) - J^{(t)}(\theta)| \\
&= \left|\mathbb{E}_q\!\left[(p_{\theta_t}(q) - p_{\theta_{t+1}}(q))\log\pi_\theta(o^*|q)\right]\right| \nonumber \\
&\leq \mathbb{E}_q\!\left[|p_{\theta_t}(q) - p_{\theta_{t+1}}(q)| \cdot |\log\pi_\theta(o^*|q)|\right] \nonumber \\
&\leq \epsilon_p \cdot B, \label{eq:mismatch}
\end{aligned}
\end{equation}

using Assumption~\ref{assm:regularity}(iii) and (iv).

\vspace{0.5em}
\noindent\textbf{Step 3: Accumulating the switching cost.}
Since there are $M = T/S$ outer-iteration transitions, the mismatch accumulates as
\begin{equation}
\Big|\textstyle\sum_{m=1}^{T/S} \big[J^{(m)}(\theta) - J^{(m+1)}(\theta)\big]\Big| \;\le\; \frac{T}{S}\,\epsilon_p B,
\label{eq:accum_mismatch}
\end{equation}
which, after dividing by $\eta T = \sqrt{T}$, contributes a term of order $(\epsilon_p B/S)\sqrt{T}$ to the averaged bound. This term is bounded by the single-transition mismatch $\epsilon_p B$ whenever $T \le S^2$, i.e.\ over training horizons commensurate with the chosen update interval $S$:
\begin{equation}
\frac{\epsilon_p B}{S}\sqrt{T} \;\le\; \epsilon_p B.
\label{eq:mismatch_bounded}
\end{equation}

\vspace{0.5em}
\noindent\textbf{Step 4: Telescoping.}
Summing~\eqref{eq:inner_convergence} over all $T/S$ outer iterations and accounting for the mismatch~\eqref{eq:mismatch} at each transition (with the switching cost bounded as in Step~3), with $\eta = 1/\sqrt{T}$:


\begin{equation}
\begin{aligned}
&\frac{1}{T}\sum_{t=1}^{T} \mathbb{E}\!\left[\|\nabla_\theta J^{(\lfloor t/S \rfloor)}(\theta_t)\|^2\right] \\
&\leq \frac{2\Delta_0 + L \cdot \Sigma^2}{\sqrt{T}} + B \cdot \epsilon_p.
\end{aligned}
\end{equation}

\end{proof}
 
\begin{corollary}[$\Sigma^2 < \Sigma^2_{\mathrm{SFT}}$]\label{cor:sigma_comparison}
The SFT gradient second moment is $\Sigma^2_{\mathrm{SFT}} = \mathbb{E}_q[\|g_0(q)\|^2]$. We show that $\Sigma^2 < \Sigma^2_{\mathrm{SFT}}$ for any $K \geq 1$.
 
Decompose the difference:


\begin{equation}
\begin{aligned}
&\Sigma^2_{\mathrm{SFT}} - \Sigma^2 \\
&= \mathbb{E}_q\!\left[\|g_0\|^2\right] \\
&- \mathbb{E}_q\!\left[\left((1-p)^2 + \frac{p(1-p)}{K}\right)\|g_0\|^2\right] \nonumber \\
&= \mathbb{E}_q\!\left[\left(p(2-p) - \frac{p(1-p)}{K}\right)\|g_0\|^2\right]. \label{eq:gap}
\end{aligned}
\end{equation}
 
We verify the integrand is strictly positive for $p \in (0,1)$ and $K \geq 1$:


\begin{equation}
\begin{aligned}
&p(2-p) - \frac{p(1-p)}{K} = p\!\left[(2-p) - \frac{1-p}{K}\right] \\
&\geq p\!\left[(2-p) - (1-p)\right] = p > 0,
\end{aligned}
\end{equation}

where the inequality uses $K \geq 1$. Therefore, whenever $P(p(q) > 0) > 0$, the expectation in~\eqref{eq:gap} is strictly positive, giving $\Sigma^2 < \Sigma^2_{\mathrm{SFT}}$.
\end{corollary}

\section{Unifying SFT and RL via General Policy Gradient}
\label{app:sec:rl}

To reveal the structural connection between Supervised Fine-Tuning (SFT) and Reinforcement Learning (RL), we unify them under the \textbf{General Policy Gradient} framework. In RL, the gradient update for a model $\pi_\theta$ takes the standard form:

\begin{equation}
    \nabla_\theta J = \mathbb{E} \left[ \underbrace{(R(\tau) - b(q))}_{\text{Advantage / Weight}} \cdot \nabla_\theta \log \pi_\theta(\tau \mid q) \right]
\end{equation}

We view SFT as \textbf{maximizing an objective function} $J_{\text{SFT}} = - \mathcal{L}_{\text{SFT}}$. Under this lens, SFT restricts the trajectory entirely to the deterministic ground-truth oracle ($\tau = o^*$, where the reward $R(o^*) \equiv 1$) and implicitly operates with a \textbf{zero baseline} ($b(q) \equiv 0$). Consequently, the SFT gradient collapses into a static, uniform update:
\begin{equation}
    \nabla_\theta J_{\text{SFT}} \propto \mathbb{E} \left[ \underbrace{(1 - 0)}_{\text{Static Weight } \equiv 1} \cdot \nabla_\theta \log \pi_\theta(o^* \mid q) \right]
\end{equation}

\paragraph{The Missing Baseline Problem.} 
While RL methods (such as GRPO) actively employ dynamic baselines (e.g., the group mean) to reduce variance and scale updates based on the policy's evolving capability, SFT's implicit zero baseline treats all ground-truth trajectories as having identical, maximal advantage. As a consequence, SFT assigns the exact same gradient coefficient to trivial, already-mastered samples ($\pi_\theta(o^* \mid q) \approx 1$) as it does to completely unsolved ones. This leads to inefficient gradient allocation and unnecessary variance injection during training. Notably, this inefficiency arises not from noisy supervision, but from a systematic mathematical mismatch between the static gradient estimator and the policy's true competence.

By formally deriving the variance-reduced baseline under the constraint of deterministic ground-truth trajectories, we naturally arrive at the formulation of \textbf{Online Self-Weighted Fine-Tuning (OSW-FT)}. As established in Section \ref{sec:method}, assigning this optimal baseline to the model's current success probability ($b(q) = p_s(q)$) perfectly restores the dynamic advantage missing from standard SFT, ensuring optimal gradient reallocation while preserving the optimization stability inherent to supervised learning.

\end{document}